\def\year{2021}\relax
\documentclass[letterpaper]{article} % DO NOT CHANGE THIS
\usepackage{aaai21}  % DO NOT CHANGE THIS
\usepackage{times}  % DO NOT CHANGE THIS
\usepackage{helvet} % DO NOT CHANGE THIS
\usepackage{courier}  % DO NOT CHANGE THIS
\usepackage[hyphens]{url}  % DO NOT CHANGE THIS
\usepackage{graphicx} % DO NOT CHANGE THIS
\usepackage{natbib}  % DO NOT CHANGE THIS AND DO NOT ADD ANY OPTIONS TO IT
\usepackage{caption} % DO NOT CHANGE THIS AND DO NOT ADD ANY OPTIONS TO IT
\usepackage[switch]{lineno}

\usepackage{microtype}
\usepackage{graphicx,subcaption}
\usepackage{booktabs} % for professional tables

\usepackage{subcaption}
\usepackage{algorithm}
\usepackage{algorithmic}

\usepackage{hyperref}

\usepackage{amssymb, amsmath, amsthm}

\newtheorem{theorem}{Theorem}
\newtheorem{definition}{Definition}
\newtheorem{lemma}{Lemma}

\newtheorem{proposition}{Proposition}

\newtheorem{assumption}{Assumption}

\title{When Does MAML Objective Have Benign Landscape?}
\author {
        Igor Molybog,\textsuperscript{\rm 1}
        Javad Lavaei \textsuperscript{\rm 1}\\
}
\affiliations {
    \textsuperscript{\rm 1} University of California at Berkeley \\
    igormolybog@berkeley.edu, lavaei@berkeley.edu
}

\begin{document}
% \linenumbers 
\maketitle

\begin{abstract}
The paper studies the complexity of the optimization problem behind the Model-Agnostic Meta-Learning (MAML) algorithm. The goal of the study is to determine the global convergence of MAML on sequential decision-making tasks possessing a common structure. We are curious to know when, if at all, the benign landscape of the underlying tasks results in a benign landscape of the corresponding MAML objective. For illustration, we analyze the landscape of the MAML objective on LQR tasks to determine what types of similarities in their structures enable the algorithm to converge to the globally optimal solution.
\end{abstract}

\section{Introduction}

Meta-learning, along with transfer learning, is a rapidly developing research area in machine learning which aims to design algorithms that gain computational advantages out of inherent similarities between learning problem instances, otherwise referred to as tasks. In this work, we study one of the most popular meta-leraning algorithms, called {\it Model-Agnostic Meta-Learning (MAML)}, which has been developed by \citet{finn2017model}. In the reinforcement learning domain, the algorithm is expected to rapidly adapt a pre-learned policy to a new task. However, it remains unclear how to measure the quality of adaptation and what tasks are suitable for the meta-learning. In the core of meta-learning, there is an optimization problem that is concerned with the expected return averaged among considered tasks. In this paper we propose to measure the performance of MAML with the optimality gap of the corresponding optimization problem. We consider a space of tasks to be suitable for meta-learning if the algorithm converges to a global optimizer of the meta-learning objective or to a point with a similar value. This approach enables distinguishing those meta-learning problems that are solvable by MAML from the other problems. Intuitively, a particular algorithm should perform satisfactorily on a set of meta-learning problems that consist of tasks united by a particular type of similarity. For the purpose of demonstration, we consider linear quadratic control problems, although our theory applies to a broad class of RL tasks with benign optimization landscape. We aim to theoretically study the global convergence properties of the original MAML algorithm on sequential decision-making tasks. In short, our findings can be summarised as follows:

\begin{itemize}
    \item Meta-Learning objective inherits a benign landscape from the objectives of the individual tasks if they are similar pointwise. As a result, the original MAML and other meta-learning algorithms that rely on local search are guaranteed to perform well on the corresponding problems.
    \item As a strongly negative result, those problems consisting of linear quadratic tasks that coincide up to a scaling of the reward function are not solvable by MAML. We propose an alternative scheme that addresses the issue with this type of similarity. 
\end{itemize}

For clarity of explanation, we investigate discrete stationary infinite-horizon decision problems and note that the generalization of the results to non-stationary and finite-horizon cases is straightforward. A stationary discrete dynamical system is described as
%the dynamics model $f_t$, where $f_t$ is specified as
\[
s_{t+1} \sim T(s_t,a_t) \, ,
\]
where $T$ is a probability distribution over the next state $s_{t+1}\in \mathbb R^d$ given the current state $s_t\in \mathbb R^d$ and action $a_t \in \mathbb R^r.$ The initial state $s_0$ is assumed to follow the distribution $T_0$. The
objective of the infinite-horizon problem is to find a control input $a_t$ minimizing the total discounted cost (the negation of the reward)
\begin{eqnarray*}
&\textrm{minimize} & \mathbb E_{s_0\sim T_0}\sum_{t=0}^\infty \gamma^t c(s_t,a_t)\\
&\textrm{subject to} & s_{t+1} \sim T(s_t,a_t) \;\; t=0,1,\ldots 
\end{eqnarray*}
where $\gamma\in (0,1]$ is the discount factor that is assumed to be $1$ for the LQR problem. 

\subsection{Linear-Quadratic Regulator}\label{sec:LQR}

One of the important examples of decision-making problems is related to the control of linear dynamical systems with a quadratic objective, referred to as {\it linear-quadratic regulator (LQR)}.
LQR and iLQR \citep{todorov2005generalized} are fundamental tools for model-based reinforcement learning. At the same time, linear-quadratic systems enjoy a rich theoretical foundation with a large number of provable guarantees, which makes them the perfect benchmark for the mathematical analysis of novel learning algorithms.

We consider the following infinite-horizon exact LQR problem: 
\begin{equation}\label{eqn:LQR_obj}
\begin{aligned}
&\textrm{minimize} & \mathbb E_{x_0}\left[\sum_{t=0}^\infty (s_t^\top Q s_t + a_t^\top R a_t )\right] \\
&\textrm{subject to} & s_{t+1} = As_t + B a_t \, ,
\end{aligned}
\end{equation}

Assuming that the matrices $A$ and $B$ are such that the optimal cost is finite, it is a classic result that the optimal control policy is deterministic and linear in the state, i.e., %, the optimal control policy has the form:
\[
a_t = -W^* s_t %\, , \textrm{where} \, \, \, K^{*} = -(B^{T} P B + R)^{-1} B^{T} P A.
\]
where $W^* \in \mathbb R^{r\times d}$ \citep{bertsekas2017dynamic}. Moreover, the matrix $W^*$ can be found from the model parameters by solving the Algebraic
Riccati Equation (ARE)
\begin{equation*}\label{eq:ARE}
P= A^{\top} P A + Q - A^{\top} PB(B^{\top} PB+R)^{-1} B^{\top} P A \, ,
\end{equation*}
and substituting the positive-definite root $P$ into 
\begin{equation*} \label{eq:K_from_P}
W^{*} = (B^{\top} P B + R)^{-1} B^{\top} P A.
\end{equation*}
This implies that in order to find the solution of LQR, it suffices to only search over deterministic policies of the form $a=-Ws$ parameterized with a matrix $W\in \mathbb R^{r\times d}.$

In an effort to build a bridge between practical RL algorithms and the optimal control theory, \citet{fazel2018global} shows that $W^{*}$ can be found by applying the policy gradient algorithm to a reformulated cost function. The LQR cost of a linear deterministic policy with respect to $W$ can be defined as
\begin{eqnarray*}
C(W) := \mathbb E_{s_0\sim T_0}\left[ \sum_{t=0}^\infty (s_t^\top Q s_t + a_t^\top R a_t) \right]
\end{eqnarray*}
where $a_t=-Ws_t$ and $s_{t+1} = (A-BW)s_{t}.$ This can be reformulated as
\[C(W) = \mathbb E_{s_0\sim T_0 }  \, s_0^\top P_W s_0 \, \]
where $P_W$ is the solution of $P_W = Q + W^\top R W + (A-BW)^\top P_W  (A-BW).$
% from which its gradient with respect to the policy can be written explicitly:
% \[
% \nabla C(W) = 2 \left( (R+ B^\top P_W B) W - B^\top P_W A\right)
% \Sigma_W
% \]
% where 
% \[
% \Sigma_W = \mathbb E_{s_0\sim T_0 } \sum_{t=0}^\infty s_t s_t^\top \, .
% \]
We only consider the cost of stable policies, and assume the cost to be infinite for unstable ones.

\subsection{Model-Agnostic Meta-Learning (MAML)}

Given a set of tasks $\mathcal T,$ each represented by an objective function $\mathcal L_{\tau},$ and a probability distribution $\mathbb P_{\mathcal T}$ over the tasks, \citet{finn2017model} proposes an algorithm for finding an initialization of the policy gradient method that allows a fast adaptation to a task through just several gradient updates. In case the task consists in regression, classification, or clusterization, $\mathcal L_{\tau}$ is a risk or an empirical risk. In case of a reinforcement learning task, $\mathcal L_{\tau}$ is the cost (the negated return) of a policy.

If the space of considered policies is parameterized via $w\in \mathcal W,$ then a single-shot MAML (which uses just one gradient update) aims to minimize the objective
\begin{equation} \label{eqn:maml_obj}
    \mathbb E_{\tau\in \mathcal T} f_{\tau}(w - \eta \nabla g_{\tau}(w))
\end{equation}
where $f_\tau$ and $g_\tau$ can be two different approximations of the objective function of the task $\tau.$ Similarly, a multi-shot version applies multiple gradient updates within the adaptation procedure. For example, if $\mathcal L_{\tau}$ is the risk of a learning problem, then $f_\tau$ can be the empirical risk conditioned on a large dataset, while $g_\tau$ is the empirical risk conditioned on a smaller dataset. The generality of this formulation will be used in the Main results Section, but for the study of LQR we will assume that the functions $\mathcal L_\tau,$ $f_\tau$ and $g_\tau$ all coincide and are equal to $C(W)$. When it is clear from the context which task is being discussed, we will omit the subscript.

Being based upon gradient descent with a constant step size, Algorithm \ref{alg:maml} is a basic version of a few-shot MAML although other versions have been developed in the literature, e.g. FO-MAML, HF-MAML \citep{fallah2019convergence}, iMAML \citep{rajeswaran2019meta}, Reptile \citep{nichol2018first} and FTML \citep{finn2019online}. We do not directly address the other formulations in the paper, but the conclusions of this work are applicable to them as well since they are created to minimize essentially the same objective function \eqref{eqn:maml_obj}
and the convergence to at least a first-order stationary point has been proven for the majority of these variants. 

An interesting modification of MAML is based on exact proximal optimization as an alternative to a gradient update. \citet{zhou2019efficient} proposes the proximal update MAML algorithm with the adaptation procedure that finds the best set of parameters in the neighborhood of initialization,
while \citet{wang2020global} proves that a similar idea can be shaped into a meta-RL algorithm that is proved to converge globally under some overparametrization assumption. However, the exact optimization in the adaptation procedure may be problematic in meta-learning setup, since adaptation is supposed to be fast, meaning that there are sharp constraints on sample complexity and the amount of computation allowed for it.

\begin{algorithm}[t]
\caption{Model-Agnostic Meta-Learning (MAML)}
\label{alg:maml}
\begin{algorithmic}[1]
\REQUIRE $p(\mathcal T)$: Probabilistic task generator
\REQUIRE $\eta$, $\beta$: Step size hyperparameters
\STATE Randomly initialize $\theta$
\WHILE{not done}
\STATE Sample batch of tasks $\mathcal T_i \sim p(\mathcal T)$
  \FORALL{$\mathcal T_i$}
 \STATE Evaluate $\nabla g_{\mathcal T_i}(w)$
 \STATE Compute adapted parameters with gradient descent: $w_i'=w-\eta \nabla_w g_{\mathcal T_i}(w)$
 \ENDFOR
 \STATE Update $w \leftarrow w - \beta \nabla_w \sum_{\mathcal T_i \sim p(\mathcal T)} f_{\mathcal T_i} ( w_i')$
\ENDWHILE
%\STATE while 
\end{algorithmic}
\end{algorithm}

% Many of these successes have relied on  algorithms such as policy gradient methods, including the
% DeepRL approaches. For these approaches, there is 
% little theoretical understanding of their efficiency, either from a
% statistical or a computational perspective. In contrast, control theory

As meta-learning seeks to improve learning performance by exploiting similarities between tasks, it is important to understand what type of similarities a particular meta-learning algorithm can take advantage of. A highly desirable feature of a meta-learning algorithm is an acceptable meta-test performance at least on the tasks it has been meta-trained on. In other words, if the set of tasks $\mathcal T$ is finite and the meta-training procedure has access to all of them, then it should succeed on the meta-testing stage. For MAML, this translates into a requirement of successful minimization of the objective \eqref{eqn:maml_obj}. \citet{finn2019online} shows that the global minimum is achieved in case the functions $f_\tau=g_\tau$ are all smooth and strongly convex. However, the objective of many practical decision-making tasks are not convex, although may possess benign landscape, like the LQR objective \eqref{eqn:LQR_obj}. Tasks with non-convex objectives are substantially harder to analyze, and thus \citet{fallah2019convergence} only shows convergence of MAML to a first-order stationary point of \eqref{eqn:maml_obj} if $f_\tau$ are non-convex. We aim to study the global convergence properties of MAML applied to tasks with non-convex objectives. 

The landscape of \eqref{eqn:LQR_obj} has been studied by \citet{fazel2018global}. They note that there exist instances of LQR that are not convex, quasi-convex, or star-convex, which means that none of the existing results on global convergence of MAML can be applied even to such a basic decision problem as LQR. However, \eqref{eqn:LQR_obj} possess benign landscape, and our study shows that this property can be transferred to \eqref{eqn:maml_obj}. 

The purpose of this paper is not to present a novel algorithm or conduct a study of its application on a specific real-world case, but rather to prove the basic properties of a popular existing algorithm for a vast variety of cases. 
Meta-Learning algorithms are supposed to capture similarities between tasks, although there is no clear way to determine whether or not a similarity has been captured. For MAML, we propose a criterion that is based on the properties of the landscape of the optimization associated with MAML.
Specifically, we declare that a version of MAML captures the similarities between the tasks in $\mathcal T$ if the objective of the algorithm on $\mathcal T$ has a benign landscape. Otherwise, the tasks in $\mathcal T$ are recognized to be too distinct for this particular version of MAML.

The advantage of this criterion is that it correlates with the computational complexity of the problem that MAML aims to solve. If the objective function has benign landscape, then the optimization problem has a low computational complexity and one can solve it by a local search method, which is implemented within MAML. If the objective does not have benign landscape, then MAML can become stuck in a spurious local minimum, which can potentially be arbitrarily worse than the optimal solution.

A drawback of this approach is that it does not allow to compare Meta-Learning algorithms against each other, since it does not take into account the computational complexity of the adaptation algorithm. 
In this paper, we do not consider this aspect because of our focus specifically on the few-shot MAML. To the best of our knowledge, this is the first paper to propose a systematic way of reasoning about the general theory of meta-adaptation between RL tasks.
We consider MAML primarily with applications to linear-quadratic systems because they are realistic and yet easy to analyze although our conclusions go far beyond this application.

\subsection*{Notation}
Given a set $\mathcal Z\subseteq \mathbb R^n$ and a point $\bar z\in \mathcal Z,$ define the {\it open neighborhood} $\mathcal{B}_{\delta}({\bar z})=\{z\in \mathcal Z| \|z-\bar z\|_2<\delta\}.$ Given a function $\ell:\mathcal Z\to \mathbb R,$ we call $\bar z\in \mathcal Z$ a {\it local minimizer} of the function $\ell$ if there exists $\delta>0$ such that $\ell(z)\ge \ell(\bar z)$ for all $z\in \mathcal B_{\delta}(\bar z).$ The value $\ell(\bar z)$ in this case is called a {\it local minimum}. A {\it global minimizer} of $\ell$ is a point $\bar z$ such that $\ell(z)\ge \ell(\bar z)$ for all $z\in \mathcal Z.$ The {\it global minimum} is a value $\min_{z\in \mathcal Z}\ell(z)$ such that $\ell(z)\ge \min_{z\in \mathcal Z}\ell(z)$ for all $z\in \mathcal Z.$ A {\it spurious local minimum} is a local minimum that is not a global minimum. Given a differentiable function $\ell:\mathcal Z\to \mathbb R,$ a {\it first-order stationary point} $\bar z\in \mathcal Z$ is such that $\nabla \ell (\bar z)=0.$ Given a twice differentiable function $\ell,$ a first-order stationary point $\bar z$ such that $\nabla \ell (\bar z) \succeq 0$ is called a {\it second-order stationary point}. Given a matrix $A\in \mathbb R^{n\times m},$ we denote its transpose as $A^\top$ and its operotor norm in the space $\ell_2$ as $\|A\|.$ For a vecor $v$ from an $\ell_2$-space, its norm is denoted with $\|v\|.$ Cardinality of the set $\mathcal T$ is denoted with $|\mathcal T|.$

\section{Main results}\label{sec:results}

In this section, we study the MAML algorithm under four different scenarios. We consider MAML applied to a single task and to several identical tasks. After that, we introduce a metric between tasks and extend the study to a number of close tasks, and, finally, we study the convergence of MAML on a large number of distant LQR tasks.

We provide theoretical results for general multi-dimensional systems, while all of the presented examples and counter-examples are on one-dimensional LQR tasks since they are easy to visualize. Note that these examples are extendable to multi-dimensional systems as well. The details on the exact tasks used for the computations are provided in the Appendix.

\subsection{Single task}
We begin by analyzing MAML applied to a singleton task set $\mathcal T.$ If MAML fails under this scenario, then there is little hope on its global convergence for the multiple-task scenario. For the single task, we rewrite the MAML objective \eqref{eqn:maml_obj} as
\[h(w) = f(w-\eta\nabla g(w))\]
where $f$ is assumed to be a continuously differentiable function and $g$ is assumed to be twice continuously differentiable. As noticed in the Introduction, the existing results on global convergence of MAML are not applicable to LQR. Figure \ref{img:maml_single_task} demonstrates an example of the MAML objective \eqref{eqn:maml_obj} applied to a single LQR task. It is non-convex and has three distinct strict local minimizers. Nevertheless, all of these three points are also global minimizers, which implies that Algorithm \ref{alg:maml} would converge to its global minimizer from almost any initial point. The minimizer in the middle corresponds to $W^*$ of the task, while the rightmost and leftmost minimizers are some points $W$ such that $W-\eta\nabla C(W) = W^*$ and $\nabla C(W)\ne 0.$ The minimizers on the both sides rely on the rapid adaptation during the meta-testing stage, which has been assumed by the creators of the algorithm. 

\begin{figure}
\centering
    \includegraphics[width=1\linewidth]{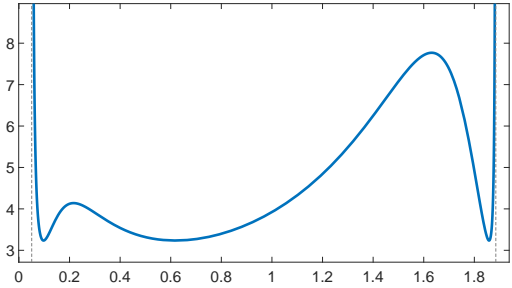}

\caption{\label{img:maml_single_task}
MAML objective \eqref{eqn:maml_obj} on a single LQR task
}
\end{figure}

This example gives rise to a hypothesis that the MAML objective for a single LQR possesses some sort of benign landscape and, more generally, that benign landscape of the cost of the underlying task results in benign landscape of the resulting MAML objective. Following \citet{josz2018theory}, we formalize the notion of benign landscape by defining the global function:
\begin{definition}
A continuous function $\ell:\mathcal Z\to \mathbb R$ is called global if every local minimizer of $\ell$ is a global minimizer.
\end{definition}
This property is also referred as having no spurious local minima and generalizes the notions of convexity, quasiconvexity, and star-convexity. As a relaxation of this property, we define $\varepsilon$-global function:
\begin{definition}
A continuous function $\ell:\mathcal Z\to \mathbb R$ is called $\varepsilon$-global if for every local minimizer $\bar z$ of $\ell$ it holds that
\[\ell(\bar z)-min_{z\in \mathcal Z}\ell(z) \le \varepsilon\]
\end{definition}
Being global is equivalent to being $0$-global. This property is more likely to be satisfied than a perfect no-spurious property for cost functions coming from real-world applications, and it still implies that the landscape is benign. These two notions of benign landscape characterize when a coercive function is easy to optimize using local optimization methods based on greedy descent. The following theorem shows that the benign landscape of the cost of the underlying task does indeed lead to a benign landscape for the resulting MAML objective.

\begin{theorem}\label{thm:global_maml}
Let $f:\mathcal W\to \mathbb R$ be global and $g:\mathcal W\to \mathbb R$ be twice continuously differentiable with $\|\nabla^2 g(w)\|\le M<\infty$ for all $w\in \mathcal W.$ If Algorithm \ref{alg:maml} with the parameter $\eta$ chosen to be smaller than $\frac{1}{M}$ converges to a local miminizer $w^\star \in \mathcal W$ of $h,$ then $w^\star$ is a global minimizer of $h.$
\end{theorem}

In the context of LQR, Theorem \ref{thm:global_maml} results in the statement below.

\begin{theorem}\label{thm:global_maml_LQR}
Let $w^\star\in \mathbb R^{r\times d}$ be the limit point of the sequence produced by Algorithm \ref{alg:maml} (MAML) with $\eta<\frac{1}{\|\nabla^2C(w^\star)\|_2}$ applied to a single LQR task, meaning that $h(W)=C(W-\eta \nabla C(W)).$ Then, $w^\star$ is the global minimizer of $h,$ which implies that $C(W-\eta \nabla C(W))$ is a global function.
\end{theorem}

Both Theorem \ref{thm:global_maml} and \ref{thm:global_maml_LQR} are proven in the appendix, and their proofs rely significantly on the following technical Lemma:
\begin{lemma}\label{thm:loc_min_open}
Let $\ell:\mathcal Z \to \mathbb R$ be an $\varepsilon$-global function, and consider a continuous map $\mathcal{F}: \mathcal W \to \mathcal Z$ with $\mathcal Z=\text{range}(\mathcal F)$ is locally open at a local minimizer $\bar w$ of $\ell\circ \mathcal F.$ Then, it holds that 
\[\ell(\mathcal F(\bar w))-\min_{w\in \mathcal W} \ell(\mathcal F(w)) \le \varepsilon\]
\end{lemma}
Theorem \ref{thm:global_maml} has implications far beyond the study of LQR. For example, Theorem 4.3 by \citet{zhang2019policy} states that, under certain conditions, the objective of $\mathcal H_2$ linear control with $\mathcal H_{\infty}$ robustness guarantee is a global function. Applying Theorem \ref{thm:loc_min_open} to this objective yields that MAML on the mixed $\mathcal H_2/\mathcal H_{\infty}$ state-feedback control design will also have no spurious local minima under the corresponding conditions.
It is also applicable in case $f$ is not a reinforcement learning objective but an objective of a regression or a classification task. It also has a straightforward extension to the multi-shot MAML and other variants of the MAML algorithm.
Theorem \ref{thm:global_maml} can also be viewed as a practical guideline for proving global convergence post-factum. If the function $f$ is global and one runs the Algorithm \ref{alg:maml} with a parameter $\eta$ that converges to a point $w^\star$ such that $\nabla h(w^\star)=0,$ then one can check whether $\nabla^2 h(w^\star)\succ0$ and $\eta<\frac{1}{\|\nabla^2 g(w^\star)\|_2}$ and if so, $w^\star$ is guaranteed to be a global minimizer of $h.$
In case $f$ is not global but its landscape has benign properties, then the following generalization takes place:
\begin{proposition}\label{prop:epsiglob}
If $f$ is $\varepsilon$-global for some $\varepsilon>0$ and $w^{\star}\in \mathcal W$ is a local minimum
% such that $\nabla h(w^\star)=0,$ $\nabla^2 h(w^\star)\succ 0$
and $\eta<\frac{1}{\|\nabla^2 g(w^\star)\|_2},$ then
\[h(\bar w) - \min_{w\in \mathcal W} h(w) \le \varepsilon\]
\end{proposition}
\begin{proof}
The mapping $\mathcal F(w) = w-\eta \nabla g(w)$ is continuously differentiable over $\mathcal W.$ The Jacobian of this mapping is $\nabla \mathcal F(w) = \mathcal I-\eta \nabla^2 g(w).$ By assumption, $\eta<\frac{1}{\|\nabla^2 g(w^\star)\|_2}$ and consequently there exists $\delta$ such that $\eta<\frac{1}{\|\nabla^2 g(w)\|_2}$ for all $w\in \mathcal B_{\delta}(w^\star).$ Similarly to the proof of Theorem \ref{thm:global_maml}, $\nabla \mathcal F(w)$ is positive definite for all $ w\in \mathcal B_{\delta}(w^\star),$ and therefore by Lemma \ref{thm:inverseFT} in the appendix, $\mathcal F\big|_{\mathcal B_{\delta}(w^\star)}$ is an open mapping and hence locally open at $w^\star.$ By Lemma \ref{thm:loc_min_open}, it means that $h(\bar w) - \min_{w\in \mathcal W} h(w) \le \varepsilon.$
% The second-order sufficient condition of local optimality is satisfied at $w^\star$ and b
\end{proof}

So far, we have shown that the benign landscape properties of MAML applied to a singleton $\mathcal T$ are inherited from the benign landscape of the objective of the task to which MAML has been applied. In particular, this holds true for the LQR tasks. 

\subsection{Several identical tasks}
Results obtained for the single-task scenario give rise to a hypothesis that the benign landscape of every individual task would help with the convergence of MAML in a multi-task setting as well, provided that all of the tasks have a similar structure.
Thus, in this part we study multiple-task learning for which the MAML has originally been designed for. Starting with some tasks that are the most similar to each other, we consider LQR tasks that coincide up to multiplication of the cost by a positive scalar. This is the highest degree of similarity one can hope to obtain between sequential decision-making problems. More precisely, the landscape features of the cost function (local and global minimizers, maximizers and saddle points) are preserved under this transformation, and therefore we refer to tasks of these types as identical henceforth.

\begin{figure}
\centering

    \begin{subfigure}[b]{0.8\linewidth}
    \includegraphics[width=\linewidth]{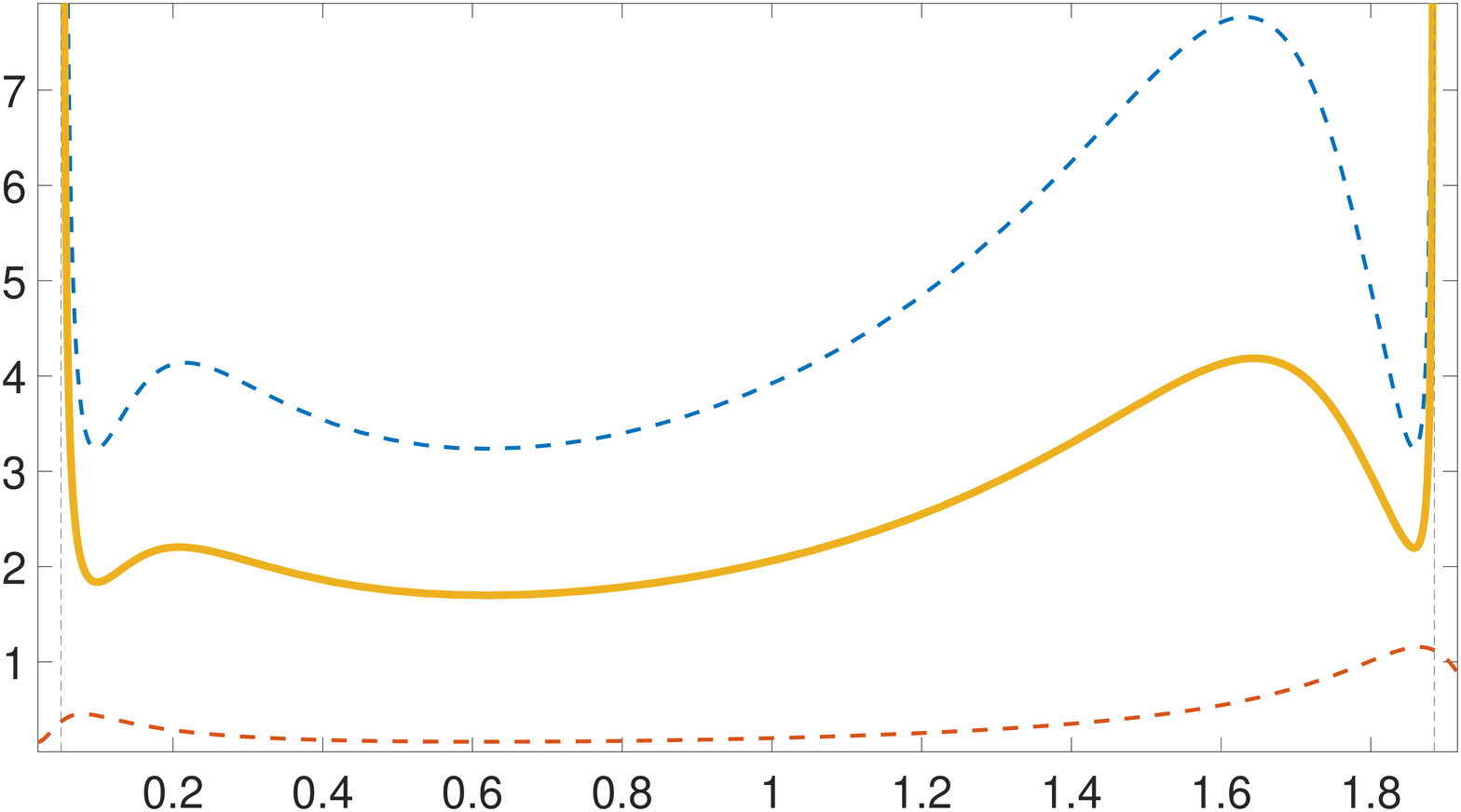}
    \caption{\label{img:two_LQR}}
  \end{subfigure}
%   \hspace*{\fill}
  \begin{subfigure}[b]{0.8\linewidth}
    \includegraphics[width=\linewidth]{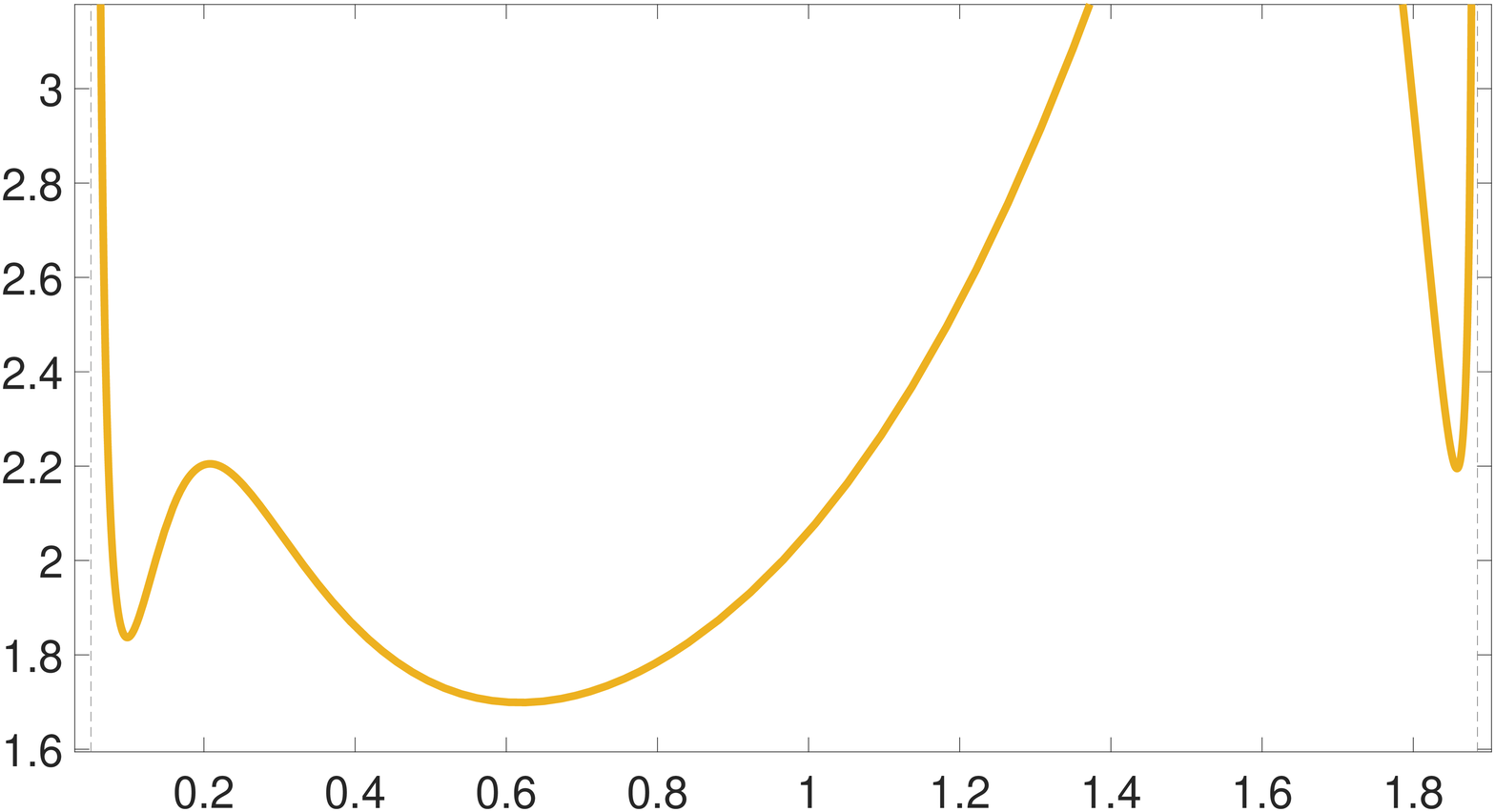}
    \caption{\label{img:two_LQR_zommed}}
  \end{subfigure}

\caption{\label{img:maml_two_tasks}
Two MAML objective functions \eqref{eqn:maml_obj} for identical LQR tasks (dashed lines) and the MAML objective for the uniform distribution among them (solid line). \ref{img:two_LQR_zommed} demonstrates spurious local minima of the solid line.
}
\end{figure}

% \begin{figure}
% \centering
%     \includegraphics[width=0.5\linewidth]{figs/five_LQR.eps}

% \caption{\label{img:maml_five_tasks}
% MAML objective \eqref{eqn:maml_obj} on five equaly probable identical LQR task with $f_i=g_i=C_i$
% }
% \end{figure}

We consider a finite set of LQR tasks with a uniform distribution among them, which allows us to reduce the MAML objective \eqref{eqn:maml_obj} to the form 
\[\frac{1}{|\mathcal T|}\sum_{\tau \in \mathcal T}C_{\tau}(W-\eta\nabla C_{\tau}(W))\]
Figure \ref{img:maml_two_tasks} demonstrates an example of the MAML objective applied to two identical LQR tasks. Although the MAML objective of each individual task is global, only one global minimizer coincides among all of them. This is the minimizer that corresponds to $W^*.$ The two minimizers on the sides are shifted and after interference they produce spurious local minima for the total objective function. Hence, we can conclude that MAML fails to capture this type of similarity between the tasks. One practical lesson to learn from this picture is that keeping the cost (or the reward) function of the considered tasks normalized may improve the quality of the solution provided by MAML. Another lesson is that the design of meta-learning algorithms may benefit from considering the analysis of identical tasks as the simplest form of common structure in the tasks.

As an alternative, we propose a modification of MAML with normalized adaptation step. It turns out that this modification manages to capture the similarity between tasks with scaled rewards. The objective function for this version of MAML under single-task scenario is $h':$

\[h'(w) = f\left(w-\eta\frac{\nabla g(w)}{\|\nabla g(w)\|}\right)\]
For which a statement similar to Proposition \ref{prop:epsiglob} holds.

\begin{proposition}\label{prop:norm_glob_single}
If $f$ is $\varepsilon$-global for some $\varepsilon>0$ and $w^{\star}\in \mathcal W$ is a local minimizer of $h(w)$ with $\eta<\frac{\|\nabla g(w^\star)\|}{\|\nabla^2g(w^\star)\|},$ then
\[h'(\bar w) - \min_{w\in \mathcal W} h'(w) \le \varepsilon\]
\end{proposition}
\begin{proof}
The mapping $\mathcal F(w) = w-\eta \frac{\nabla g(w)}{\|\nabla g(w)\|}$ is continuously differentiable over $\mathcal W.$ The Jacobian of this mapping is 
\[\nabla \mathcal F(w) = \mathcal I-\eta\left[\frac{\nabla^2g}{\|\nabla g\|^3}\left(\|\nabla g\|^2I-\nabla g\nabla g^T\right)\right]\]
Observe that 
\begin{align*}
    \left\|\frac{\nabla^2g}{\|\nabla g\|^3}\left(\|\nabla g\|^2\mathcal I-\nabla g\nabla g^T\right)\right\|\le\\ \frac{\|\nabla^2g\|}{\|\nabla g\|^3}\left\|\|\nabla g\|^2 \mathcal I-\nabla g\nabla g^T\right\|
\end{align*}
Since $\nabla g\nabla g^T$ is a rank-one matrix, the largest eigenvalue of $\|\nabla g\|^2I-\nabla g\nabla g^T$ must be equal to $\|\nabla g\|^2$ and therefore
\[\left\|\|\nabla g\|^2\mathcal I-\nabla g\nabla g^T\right\| = \|\nabla g\|^2\] 
Hence, $\eta<\frac{\|\nabla g(w^\star)\|}{\|\nabla^2g(w^\star)\|}$ is a sufficient to conclude that there exists $\delta$ such that $\nabla \mathcal F(w)$ is positive definite for all $ w\in \mathcal B_{\delta}(w^\star),$ and therefore by Lemma \ref{thm:inverseFT}, $\mathcal F\big|_{\mathcal B_{\delta}(w^\star)}$ is an open mapping and thus locally open at $w^\star.$ By Lemma \ref{thm:loc_min_open}, it follows that $h'( w^\star) - \min_{w\in \mathcal W} h'(w) \le \varepsilon.$
\end{proof}

In general, the objective of this modification of MAML can be written as \[\mathbb E_\tau f\left(w-\eta\frac{\nabla g_\tau(w)}{\|\nabla g_\tau(w)\|}\right)\] and its first-order stationary point can be found by utilizing the gradient descent with Armijo rule as noticed by \citet{bertsekas1998nonlinear} in Proposition 1.2.1.
The claim that MAML with normalized gradient step inherits the benign landscape from the identical tasks can be formulated for LQR tasks as follows:

\begin{theorem}
Let the normalized-gradient version of MAML be applied to $k$ LQR tasks $\{(A_i, B_i, Q_i, R_i)\}_{i=1}^k$ with scaled dynamics and rewards, meaning that there exist $\alpha_1, \ldots, \alpha_k>0$ and $\beta_1, \ldots, \beta_k>0$ such that 
\[A_1=\cdots=A_k;\] 
\[B_1=\cdots=B_k;\]
\[\alpha_1Q_1=\cdots=\alpha_kQ_k;\] 
\[\alpha_1R_1=\cdots=\alpha_kR_k,\]
and $h(W)=\sum_{i=1}^k\omega_iC_i\left(W-\eta \frac{\nabla C_i(W)}{\|\nabla C_i(W)\|}\right).$ Let $w^\star\in \mathbb R^{r\times d}$ be a local minimal point of $h(W)$ with $\eta<\frac{\|\nabla C_i(w^\star)\|}{\|\nabla^2C_i(w^\star)\|}$ for some $i\in \{1\ldots k\}.$ Then, $w^\star$ is the global minimizer of $h.$
\end{theorem}
\begin{proof}
From the construction of the function $C_i,$ it follows that for all $W\in \mathbb R^{r\times d}$ there exists $C(W)$ such that \[C(W)=\frac{C_1(W)}{\alpha_1}=\cdots=\frac{C_k(W)}{\alpha_k}\]
Consequently, $\nabla C(W) = \frac{\nabla C_1(W)}{\alpha_1}=\cdots=\frac{\nabla C_k(W)}{\alpha_k}$ and $\nabla^2 C(W) = \frac{\nabla^2 C_1(W)}{\alpha_1}=\cdots=\frac{\nabla^2 C_k(W)}{\alpha_k}$ for all $i\in \{1\ldots k\}.$
Hence, \[h(W) = \left[\sum_{i=1}^kw_i\alpha_i\right]C\left(W-\eta \frac{ \nabla C(W)}{ \|\nabla C(W)\|}\right)\]
Since $w^\star$ is a local minimum, by Proposition \ref{prop:norm_glob_single}, we conclude that $w^\star$ is also a global minimum.
\end{proof}

\subsection{Several similar tasks}

Moving forward, it is desirable to consider a different type of similarity between different tasks in the MAML setting. Therefore, we study the landscape of MAML on those tasks that are similar to each other in terms of the norm of the difference between parameters. For LQR, the parameters are the matrices $A, B, Q$ and $R.$ Our result states that the benign landscape of the underlying tasks carries over to the MAML objective if the tasks are sufficiently close to each other. For LQR, it can be put in the form of the following informal proposition:
\begin{proposition}\label{thm:close_LQR}
For every $\varepsilon>0$ and $k\in \mathbb N,$ there exists $\delta>0$ such that the MAML objective \eqref{eqn:maml_obj} is $\varepsilon$-global for almost any set $\mathcal T$ of $k$ LQR tasks defined through the parameters $\{(A_i, B_i, Q_i, R_i)\}_{i=1}^k$ such that for all $i,j\in \{1\ldots k\}$
\[\|A_i-A_j\|+\|B_i-B_j\|+\|Q_i-Q_j\|+\|R_i-R_j\|\le \delta\quad \]
\end{proposition}
The formal statement along with the proof of the result and additional discussions are provided in the appendix. Intuitively, as long as the dependence of the cost of a single task on the parameters is continuous, for a set of tasks that are close to each other, the multi-task landscape of MAML remains close to the landscape of MAML for just one of them. However, for this reasoning to hold true, we would need to determine the continuity properties of a manifold of parametric $\varepsilon$-global functions. In order to do that, we make an assumption below.

\begin{assumption}\label{asm:contin}
Let $\ell:\mathcal X\times \mathbb R^{m}\to \mathbb R$ with a compact set $\mathcal X\subset \mathbb R^n$ be a twice continuously differentiable function with respect to $x\in \mathcal X$ with $\ell, \nabla_x\ell$ and $\nabla^2_{xx}\ell$ being continuous with respect to $t\in \mathbb R^m$. Assume that $\ell(\cdot,t)$ has a finite number of first-order stationary points for all $t\in \mathbb R^m$ and that the Hessian is non-singular ($\text{det}[\nabla_{xx}^2\ell(x,t)]\ne 0$) for all $t\in \mathbb R^m$ and all $x\in \mathcal{X}$ such that $\nabla_t\ell(x,t)=0.$
\end{assumption}
The following Theorem determines the continuity property of $\varepsilon$-globality over the manifold of parametric functions satisfying Assumption \ref{asm:contin}. We prove it here as an important standalone result.
\begin{theorem}\label{thm:global_perturb}
If for some $\bar t$ and $\varepsilon>0$ the function $\ell(\cdot, \bar t)$ satisfying Assumption \ref{asm:contin} is $\varepsilon$-global, then for any $\varepsilon'>0$ such that $\varepsilon'>\varepsilon$ there exists $\delta>0$ for which the function $\ell(\cdot, t)$ is $\varepsilon'$-global for all $t\in \mathcal B_{\delta}(\bar t)$.
\end{theorem}
\begin{proof}[Proof of Theorem \ref{thm:global_perturb}]
We prove the theorem in three steps. Step 1: $\ell(\cdot, t)$ has stationary points in the neighoborhoods of the stationary points of $\ell(\cdot, \bar t).$ Step 2: the types of the stationary points coincide. Step 3: There are no stationary points outside of the considered neighborhoods.

A finite number of stationary points means that all of them are isolated. Consider a stationary point $\bar x$ of $\ell(\cdot, \bar t).$ By Lemma \ref{thm:gift}, there exist $\psi>0$ and $\phi>0$ such that for all $t\in\mathcal B_\psi(\bar t)$ there exists a unique $x(t)\in \mathcal B_{\phi}(\bar x)$ with the property that $\nabla_x \ell(x(t), t) = 0.$ This implies that for every function $\ell(\cdot, t)$ with $t\in\mathcal B_\psi(\bar t)$ there is a unique stationary point over $\mathcal B_{\phi}(\bar x).$ 

Note that for any $v\in \mathbb R^n$ the function $h(x, t, v)=v^\top[\nabla_{xx}^2\ell(x, t)]v$ is continuous in both $x$ and $t.$ By the assumption of the theorem, $\text{det}[\nabla_{xx}^2\ell(\bar x,\bar t)]\ne 0,$ and therefore $\nabla_{xx}^2\ell(\bar x,\bar t)\succ 0$ if $\bar x$ is a local minimum, $\nabla_{xx}^2\ell(\bar x,\bar t)\prec 0$ if it is a local maximum, and $\nabla_{xx}^2\ell(\bar x,\bar t)$ indefinite if it is a saddle point. In each of these cases, we describe how to find values $\delta'$ and $\phi'$ such that $\ell$ has a bounded value of local minima over $\mathcal B_{\phi'}(\bar x)\times\mathcal B_{\delta'}(\bar t).$

{\it \underline{Case 1:}} $\bar x$ is a local minimum. The value of $h(\bar x, \bar t, v)$ is positive for all $v\in \mathbb R^n\backslash \{0\}.$ By continuity, there exist $\psi'>0$ and $\phi'>0$ such that $\psi'<\psi$ and $\phi'<\phi$ and $h(x, t, v)>0$ for all $x\in \mathcal B_{\phi'}(\bar x),$ $t\in \mathcal B_{\psi'}(\bar t)$ and $v\in \mathbb R^n\backslash \{0\}.$ This way, $\nabla_{xx}^2\ell(x(t),t)\succ 0$ and therefore $x(t)$ is a local minimum of $\ell(\cdot, t).$ By continuity of $\ell(x, t)$ with respect to $x$ and $t,$ there exists $\delta'>0$ such that $\delta'<\psi'$ and $|\ell(x(t), t)-\ell(\bar x, \bar t)|< \frac{\varepsilon'-\varepsilon}{2}$ for all $t\in \mathcal B_{\delta'}(\bar t).$

{\it \underline{Case 2:}} $\bar x$ is a local maximum. The value of $h(\bar x, \bar t, v)$ is negative for all $v\in \mathbb R^n\backslash \{0\}.$ By continuity, there exist $\psi'>0$ and $\phi'>0$ such that $\psi'<\psi$ and $\phi'<\phi$ and $h(x, t, v)<0$ for all $x\in \mathcal B_{\phi'}(\bar x);$ $t\in \mathcal B_{\psi'}(\bar t)$ and $v\in \mathbb R^n\backslash \{0\}.$ This way, $\nabla_{xx}^2\ell(x(t),t)\prec 0$ and therefore $x(t)$ is a local maximum of $\ell(\cdot, t).$ In this case, we assign to the point $\bar x$ the value $\delta' = \psi'.$

{\it \underline{Case 3:}} $\bar x$ is a saddle point. There exist $v\in \mathbb R^n$ and $u\in \mathbb R^n$ such that $h(\bar x, \bar t, v)>0$ and $h(\bar x, \bar t, u)<0.$ By continuity, there exist $\psi'>0$ and $\phi'>0$ such that $\psi'<\psi$ and $\phi'<\phi$ yet $h(x, t, v)>0$ and $h(x, t, u)<0$ for all $x\in \mathcal B_{\phi'}(\bar x),$ $t\in \mathcal B_{\psi'}(\bar t)$ and $v\in \mathbb R^n\backslash \{0\}.$ This way, $\nabla_{xx}^2\ell(x(t),t)$ is indefinite and therefore $x(t)$ is a saddle point of $\ell(\cdot, t).$ In this case, we assign to the point $\bar x$ the value $\delta' = \psi'.$

As a result, having selected a single stationary point $\bar x$ of $\ell(\cdot, \bar t),$ we can find $\delta'$ and $\phi'$ such that all the stationary points of $\ell(\cdot, t)$ for $t\in \mathcal B_{\delta'}(\bar t)$ are of the same type as $\bar x,$ and in case they are local minimizers, they have a value that is not too different from $\ell(\bar x, \bar t).$ One can repeat this argument for all the stationary points and therefore form sets of numbers $\{\delta'_{i}\}_{i=1}^N$ and $\{\phi'_{i}\}_{i=1}^N,$ where $i$ corresponds to the index of each of the $N$ stationary points $\bar x_i$ of $\ell(\cdot, \bar t).$ 

Consider the set $\mathcal Y = \mathcal X\backslash\cup_{i\in \{1\ldots N\}}\mathcal B_{\phi'_i}(\bar x_i).$ It is a compact set as a compact set minus an open set and $\|\nabla_x \ell(x,\bar t)\|>0$ for all $x\in \mathcal Y.$ Since $\|\nabla_x \ell(x,\bar t)\|$ is continuous in $x$ over a compact set, it is uniformly continuous and it reaches its lower bound, meaning that there exists $\xi>0$ such that $\|\nabla_x \ell(x,\bar t)\|\ge \xi$ for all $x\in \mathcal Y.$ By continuity of $\|\nabla_x \ell(x,t)\|$ with respect to $t,$ there exists $\delta''$ such that $\|\nabla_x \ell(x,t)\|>\frac{\xi}{2}$ over $\mathcal Y\times \mathcal B_{\delta''}(\bar t)$ and therefore there are no stationary points of $\ell$ over $\mathcal Y\times \mathcal B_{\delta''}(\bar t).$

We select $\delta = \min[\{\delta'_i|i\in \{1\ldots N\}\}\cup \{\delta''\}]$ and observe that for all $t\in \mathcal B_{\delta}(\bar t)$ the only local minimizers of $\ell(\cdot, t)$ are those close to the local minimizers of $\ell(\cdot, \bar t).$ In {\it Case 1}, the corresponding $\delta'$ was selected such that, given a local minimizer $x(t)$ of $\ell(\cdot, t)$ that is neighboring a local minimizer $\bar x$ of $\ell(\cdot, \bar t)$ and a global minimizer $x'(t)$ of $\ell(\cdot, t)$ that is neighboring a local minimizer $\bar x'$ of $\ell(\cdot, \bar t),$ it holds that 
\begin{align*}
\ell(x(t), t)-\min_x \ell(x, t)= \ell(x(t), t)- \ell(x'(t), t)&=\\
\ell(x(t), t)-\ell(\bar x, \bar t)+\ell(\bar x, \bar t)-\ell(\bar x', \bar t)+\\+\ell(\bar x', \bar t)-\ell(x'(t), t)&\le\\ |\ell(x(t), t)-\ell(\bar x, \bar t)|+|\ell(\bar x, \bar t)-\ell(\bar x', \bar t)|+\\+|\ell(\bar x', \bar t)-\ell(x'(t), t)|&<\\ \frac{\varepsilon'-\varepsilon}{2}+\varepsilon+\frac{\varepsilon'-\varepsilon}{2}&=\varepsilon'    
\end{align*}

\end{proof}

Proposition \ref{thm:close_LQR} goes beyond the uniform distribution and holds for any distribution on a finite number of tasks. An example of the MAML objective for five similar LQR tasks is demonstrated in Figure \ref{img:maml_norm}. Similarly to the results for single-task scenario, the statements of this section are generalizable to other tasks with benign optimization landscape. In the end, we conclude that MAML is able to capture the common structure among different tasks given through similar values of the parameters.

\subsection{A number of tasks with common dynamics}

In general, distant tasks may produce a wide variety of landscapes for the MAML objective. However, achieving benign landscape for several non-identical tasks with different norms is still possible. Figure \ref{img:maml_distnat} demonstrates the MAML objective for several LQR tasks that share the same dynamics but have different cost functions. Increasing the number of considered tasks improves the features of the landscape of the total MAML objective. 
In the eleven-task scenario, MAML seems to learn the mean of the optimal policies $W^*$ for the tasks, which means that instead of learning to rapidly adapt, MAML would learn the average policy by effectively finding the minimum of $\frac{1}{|\mathcal T|}\sum_{\tau \in \mathcal T}C_{\tau}(W).$
However, for the two-task scenario, the global minimizer of the MAML objective appears to correspond to a policy $W$ that has $\nabla C_{\tau}(W)$ far from zero for every considered $\tau$, and thus the algorithm that converged to that point would learn to adapt to a task during the meta-testing phase. 

In general, that the average of a large number of functions that each has a global minimizer in a small region of the domain would end up being almost global itself with the minimizer in the same region. Therefore, in practice, increased number of tasks for meta-training may improve the properties of the landscape of the MAML objective and assure convergence to the global minimum.

\begin{figure}
\centering
 
 \begin{subfigure}[b]{0.7\linewidth}
    \includegraphics[width=\linewidth]{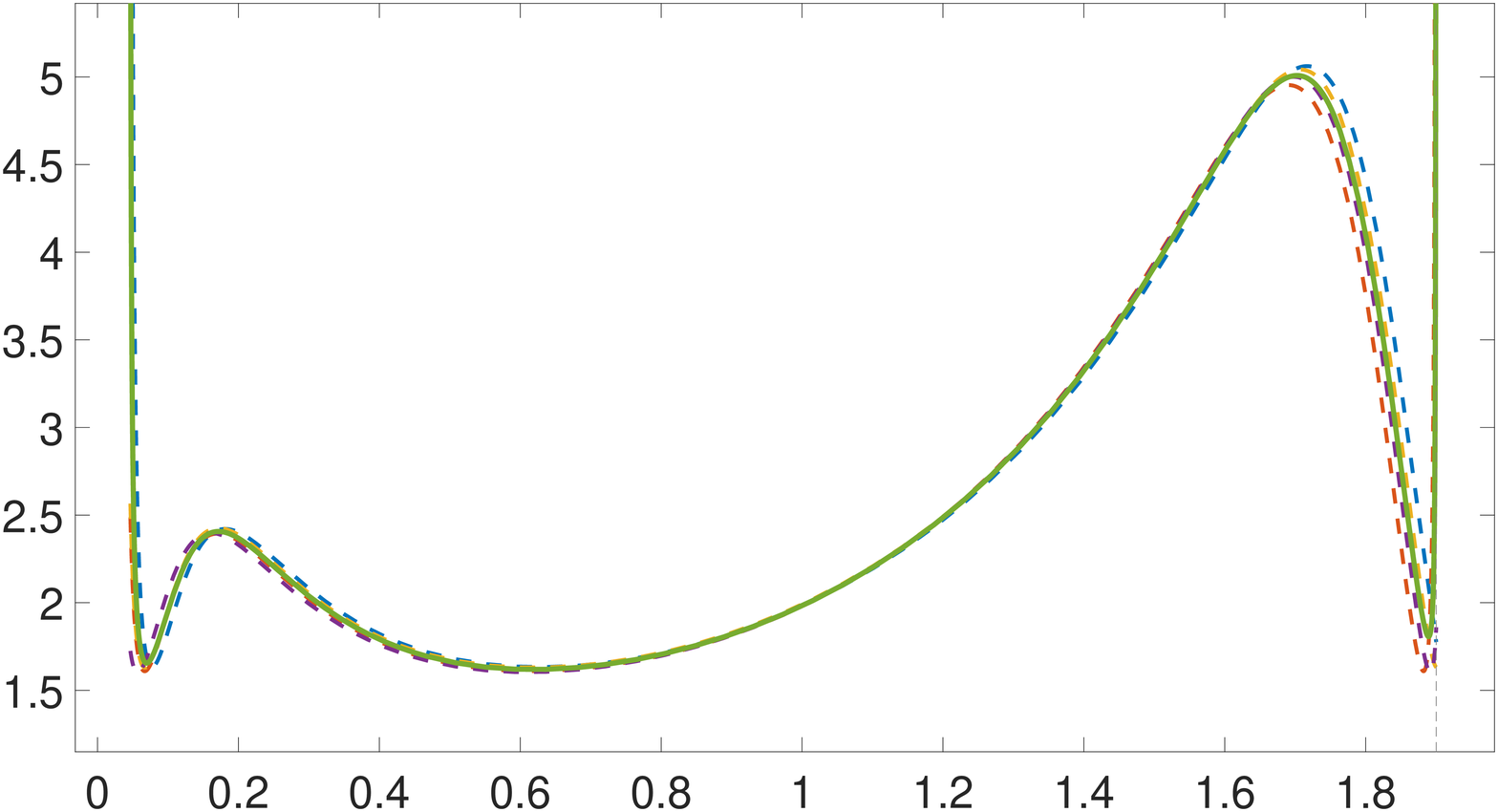}
    \caption{\label{img:maml_norm}}
  \end{subfigure}
    \begin{subfigure}[b]{0.4\linewidth}
    \includegraphics[width=\linewidth]{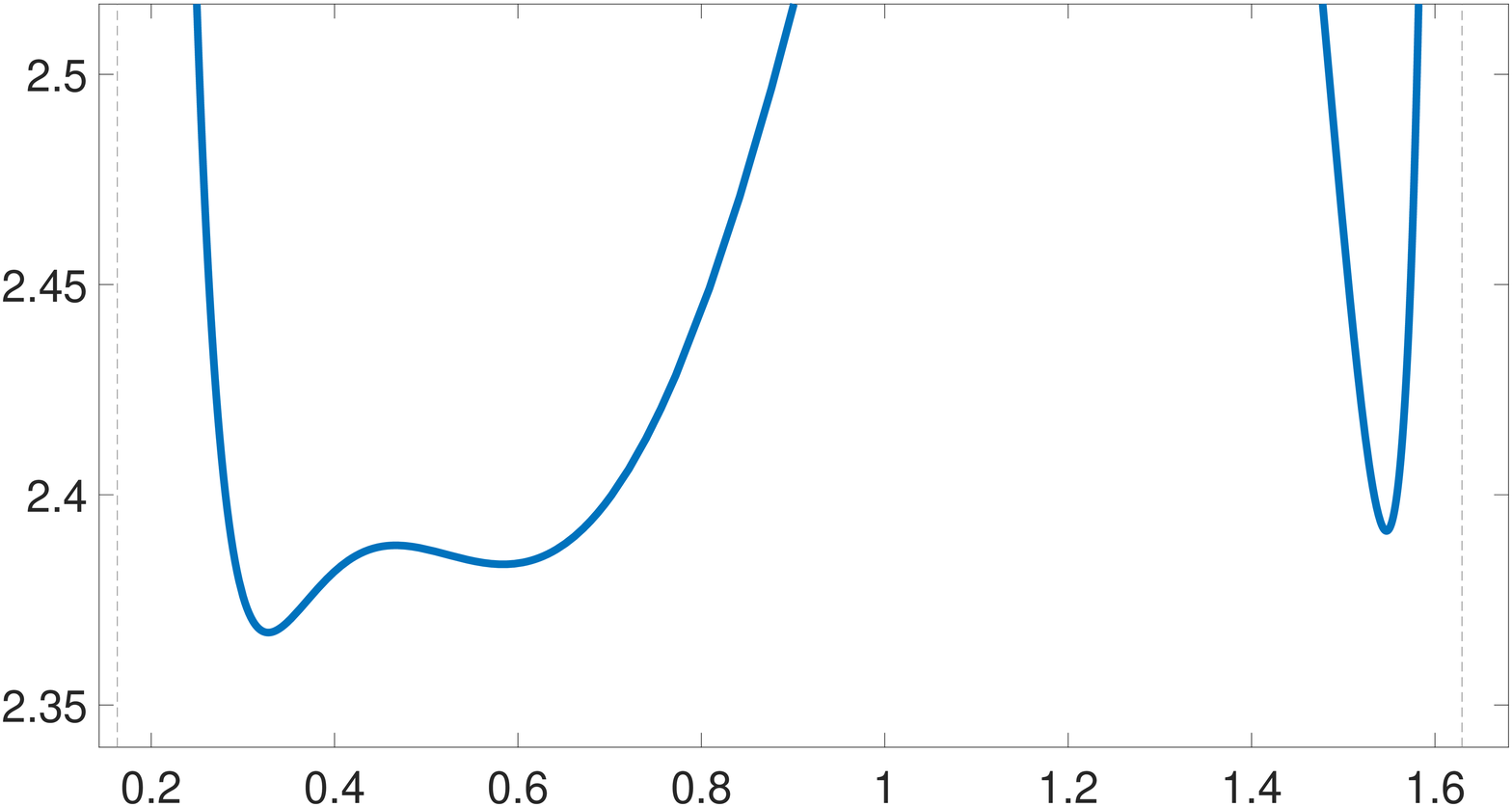}
    \caption{\label{img:distant_afew}}
  \end{subfigure}
%   \hspace*{\fill}
  \begin{subfigure}[b]{0.4\linewidth}
    \includegraphics[width=\linewidth]{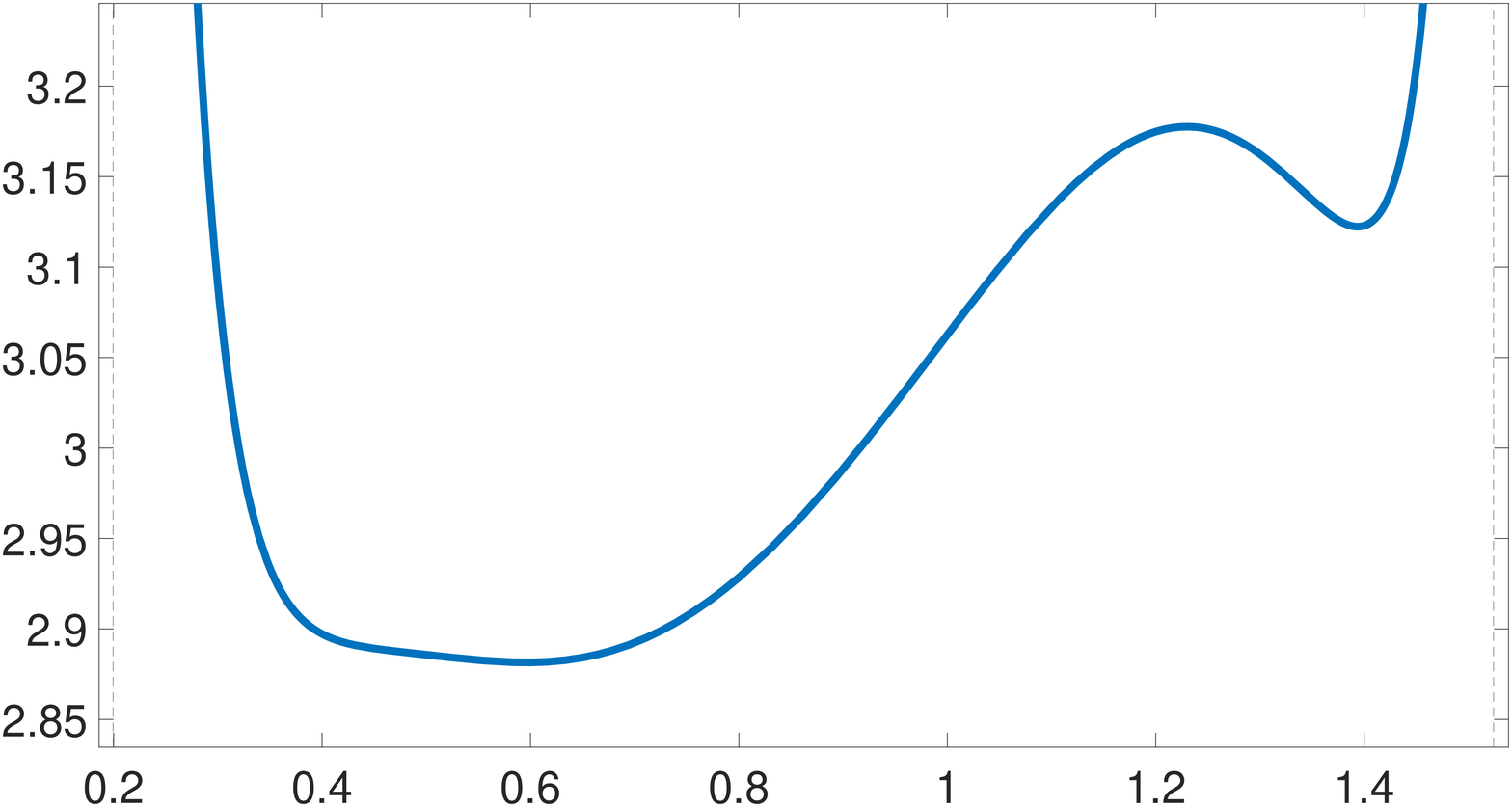}
    \caption{\label{img:distant_several}}
  \end{subfigure}
\begin{subfigure}[b]{0.4\linewidth}
    \includegraphics[width=\linewidth]{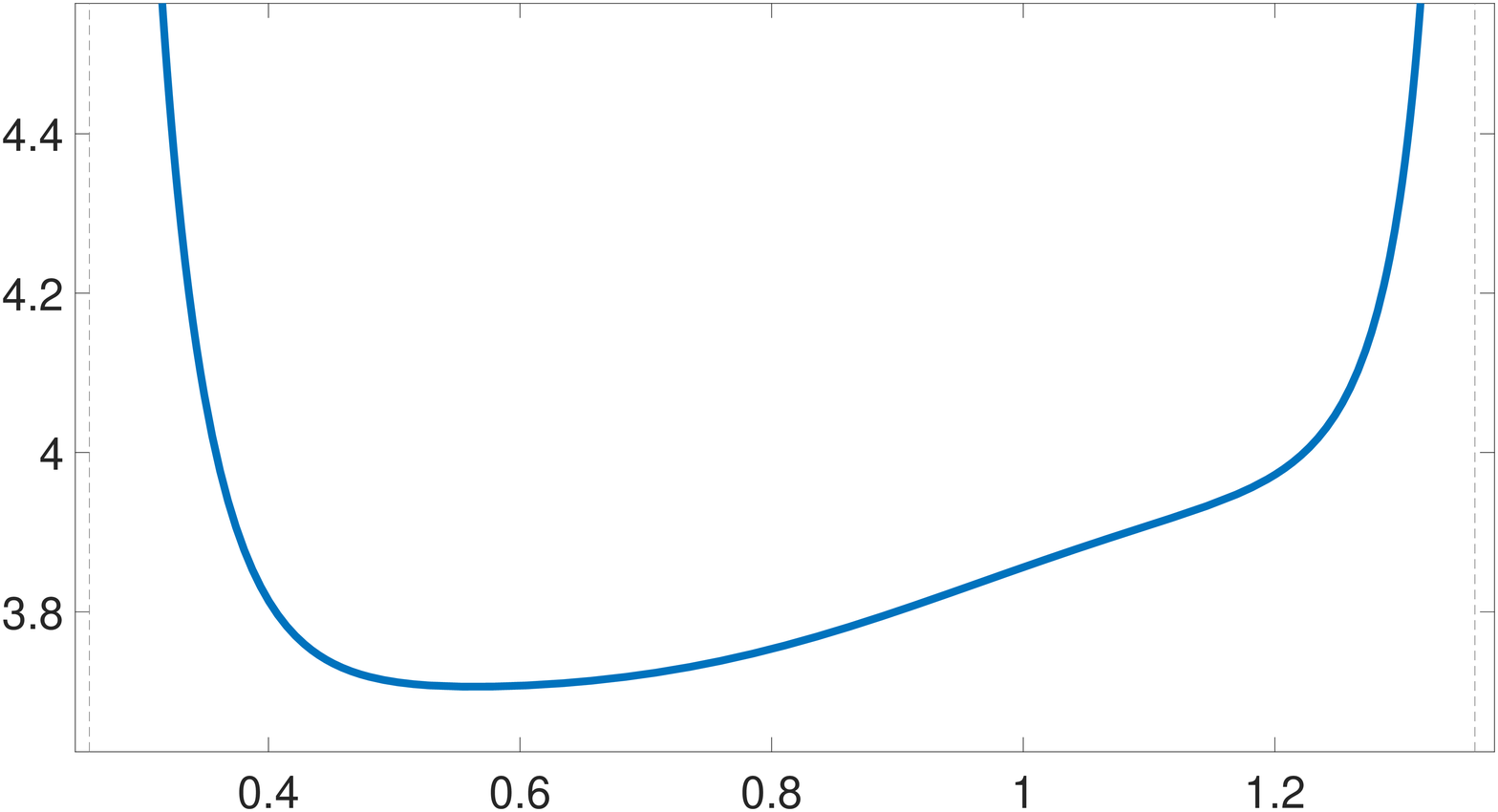}
    \caption{\label{img:distant_many}}
  \end{subfigure}

\caption{\label{img:maml_distnat}
Five MAML objective functions \eqref{eqn:maml_obj} for LQR tasks with similar values of $A,$ $B,$ $Q$ and $R$ (dashed lines) and the MAML objective for the uniform distribution among them (solid line) {\bf (plot \ref{img:maml_norm})}.
MAML objective \eqref{eqn:maml_obj} for the uniform distribution among two {\bf (plot \ref{img:distant_afew})}, five {\bf (plot \ref{img:distant_several})} and eleven {\bf (plot \ref{img:distant_many})} different LQR tasks that share the same dinamics $A$ and $B$ but have different cost matrices $Q$ and $R.$
}
\end{figure}

\section{Conclusion}

In the paper, we studied the objective \eqref{eqn:maml_obj} of the MAML algorithm constructed for different spaces of tasks with some common structure. We have shown that the MAML objective inherits the benign optimization landscape from the underlying tasks if they are similar pointwise and that this desirable property fails to hold for the tasks that coincide up to a multiplication of the objective by a positive value. We proposed a modification of MAML that does not possess this drawback.

During this study of the MAML we used LQR as the primary example, although a number of general results on benign landscapes were discovered, including continuity property for, the newly introduced notion of $\varepsilon$-global function and its properties with respect to convex combinations that are described in the appendix.

\bibliography{refs}

\newpage
$~$
\newpage

\section*{Appendix}

\subsection*{Proof of Theorem \ref{thm:global_maml}}

We will use the properties of the composition operator with open maps to prove the above theorem. First, we need to introduce some preliminary results.
Consider an open subset $\mathcal W$ of a finite-dimensional vector space and a continuous mapping $\mathcal{F}: \mathcal W \to \mathcal Z$ where $\mathcal Z=\text{range}(\mathcal F).$

% \begin{definition}
% The mapping $\mathcal{F}: \mathcal{W} \rightarrow \mathcal{Z}$ is said to be open, if for every open set $U \in \mathcal{W}$, $\mathcal{F}(U)$ is (relatively) open in $\mathcal{Z}$.
% \end{definition}
% and its local analog
\begin{definition}
A mapping $\mathcal{F}: \mathcal W \to \mathcal Z$ with $\mathcal Z=\text{range}(\mathcal F)$ is said to be locally open at ${w}$ if for every $\epsilon >0$, there exists $\delta>0$ such that $\mathcal{B}_{\delta}\big(\mathcal{F}({w})\big) \, \subseteq \, \mathcal{F}\big(\mathcal{B}_{\epsilon}({w})\big)$.
\end{definition}
\begin{definition}
A mapping $\mathcal{F}: \mathcal W \to \mathcal Z$ with $\mathcal Z=\text{range}(\mathcal F)$ is said to be open if $\mathcal{F}(U)$ is (relatively) open in $\mathcal{Z}$ for every open set $U \in \mathcal{W}$.
\end{definition}

A mapping $\mathcal F$ is open if and only if it is locally open at every point of its domain.
The following lemma allows us to analyze the local landscape of the composition of a function with a locally open map.

\begin{lemma}[Observation 1 of \citet{nouiehed2018learning}]\label{thm:loc_open}
Suppose that the continuous map $\mathcal{F}: \mathcal W \to \mathcal Z$ with $\mathcal Z=\text{range}(\mathcal F)$ is locally open at $\bar{w}$. If for some $\ell:\mathcal Z \to \mathbb R,$ the point $\bar{w}$ is a local minimum of  $\ell(\mathcal{F}({w}))$ over $\mathcal{W}$, then $\bar{z}=\mathcal{F}(\bar{w})$ is a local minimum of  $\ell({z})$ over $\mathcal{Z}$.
\end{lemma}

 The above observation applied to an $\varepsilon$-global function leads to Lemma \ref{thm:loc_min_open}.

\begin{proof}[Proof of Lemma \ref{thm:loc_min_open}]
By Lemma \ref{thm:loc_open}, $\bar z = \mathcal{F}(\bar w)$ is a local minimizer of $\ell(z).$ Since $\ell$ is $\varepsilon$-global, $\ell(\bar z)-\min_{z\in \mathcal Z}\ell(z)\le \varepsilon.$ Due to $\text{range}(\mathcal F) = \mathcal{Z},$ it follows that $\ell(\mathcal F(\bar w))-\min_{w\in \mathcal W} \ell(\mathcal F(w)) \le \varepsilon.$
\end{proof}

The Inverse Function Theorem provides us with the following corollary, which states that mappings with a nonsingular Jacobian are open.

\begin{lemma}[Theorem 9.25 in \citet{rudin1964principles}]\label{thm:inverseFT}
Let $\mathcal F$ be a continuously differentiable mapping of an open set $\mathcal W\subset \mathbb R^{n}$ into $\mathbb R^n.$ Suppose that the Jacobian $\nabla \mathcal F(x)$ is a nonsingular matrix for all $x\in \mathcal W$. Then, for every open subset $V$ of $\mathcal W$ the set $\mathcal F(\mathcal W)$ is open. In other words, $\mathcal F: \mathcal W\to\text{range}(\mathcal F)$ is an open map.
\end{lemma}

Now, we have defined all the necessary machinery for the proof of Theorem \ref{thm:global_maml}.

\begin{proof}
The mapping $\mathcal F(w) = w-\eta \nabla g(w)$ is continiously differentiable over $\mathcal W.$ The Jacobian of this mapping is $\nabla \mathcal F(w) = \mathcal I-\eta \nabla^2 g(w),$ where $\mathcal I$ is the identity matrix. By assumption, we have $\eta< \frac{1}{M},$ and consequently the following inequality holds for all $v\in \mathbb R^n\backslash\{0\}:$ 
\[v^\top[\mathcal I-\eta \nabla^2 g(w)]v\ge (1-\eta \|\nabla^2 g(w)\|)v^\top v>0\] 
This means that $\nabla \mathcal F(w)$ is positive definite for all $w\in \mathcal W,$ and hence is non-singular. By Lemma \ref{thm:inverseFT}, $\mathcal F$ is an open mapping.

The function $f$ is global by assumption, and by Lemma \ref{thm:loc_open} the composition of a global function with an open map is global. Thus, $h$ is global as a composition of $f$ and $\mathcal F,$ and consequently each of its local minimizers is a global minimizer.
\end{proof}

\subsection*{Proof of Theorem \ref{thm:global_maml_LQR}}
\begin{proof}
It is shown in \citet{fallah2019convergence} that MAML converges to a first-order stationary point of a smooth nonconvex loss. Therefore, if the limit point $w^{\star}$ exists, it must be a stationary point of $C(W-\eta \nabla C(W)).$ Consider the first-order stationarity condition for $w^\star:$ 
\begin{align*}
    0=&\nabla C(w^\star-\eta \nabla C(w^\star)) =\\ =&[\mathcal I-\eta \nabla^2C(w^\star) ]^\top\nabla C(W)\big|_{W=w^\star-\eta\nabla C(w^\star)}
\end{align*}
Similarly to the proof of Theorem \ref{thm:global_maml}, $\mathcal I-\eta \nabla^2C(w^\star) $ is a full-rank marix, meaning that $w^\star$ is a first-order stationary point for the MAML objective if and only if $\nabla C(W)\big|_{W=w^\star-\eta\nabla C(w^\star)} = 0.$ 

Theorem 7 of \citet{fazel2018global} states that the Gradient descent algorithm finds an $\varepsilon$-approximation of the global optimum of $C(W)$ in polynomial time for any initial point with a finite value. This directly implies that all first-order stationary points of $C(W)$ are the global minimizers of $C(W)$ because otherwise we could initialize the Gradient descent algorithm at a stationary point and since it converges to the point of initialization, it will lead to a contradiction. Being a first-order stationary point is a sufficient condition of local minimality, and hence guaranties that $C(W)$ is a global function. By Theorem \ref{thm:global_maml}, the MAML objective is global for a sufficiently small $\eta$.

Since $w^\star-\eta\nabla C(w^\star)$ is a first-order stationary point of $C,$ it is a global minimizer of $C$ and since $\min_{W}C(W)\le \min_{W}C(W-\eta\nabla C(W)),$ the point $w^{\star}$ is the global minimizer of the MAML objective.
\end{proof}

\subsection*{Proof of Proposition \ref{thm:close_LQR}}

\begin{lemma}[Theorem 9.28 of the book by \citet{rudin1964principles} (Generalized implicit function theorem)]\label{thm:gift}
Let $\mathcal F$ be a continiously differentiable mapping of an open set $\mathcal E\subset \mathbb R^{n+m}$ into $\mathbb R^n$ such that $\mathcal (\bar x,\bar t)=0$ for some point $(\bar x, \bar t)\in \mathcal E.$ Assume that $\nabla_{x}\mathcal F(\bar x, \bar t)$ is a non-singular matrix. Then, there exist open sets $U\subset \mathbb R^{n+m}$ and $V\subset \mathbb R^m,$ with $(\bar x, \bar t)\in U$ and $\bar t\in V,$ having the following properties
\begin{enumerate}
\item For all $t\in V$ there exists a unique $x=x(t)$ such that $(x, t)\in U$ and $\mathcal F(x,t)=0.$
\item The function $x(t)$ is a continiously differentiable mapping of $V$ into $\mathbb R^n,$ $x(\bar t) = \bar x$ and $\nabla_t x(t) = -[\nabla_x\mathcal F(x(t), t)]^{-1}\nabla_t\mathcal F(x(t), t)$ for all $t \in V.$ 
\end{enumerate}
\end{lemma}

\begin{lemma}\label{thm:global_convcomb}
Given $\lambda_i>0$ and $\sum_{i=1}^k\lambda_i = 1,$
if for some $\bar t$ and $\varepsilon>0$ the function $\ell(\cdot, \bar t)$ satisfying Assumption \ref{asm:contin} is $\varepsilon$-global, then for any $\varepsilon'>0$ such that $\varepsilon'>\varepsilon$ there exists $\delta>0$ for which the convex combination $\lambda_1\ell(\cdot, t_1)+\cdots+\lambda_k\ell(\cdot, t_k)$ is $\varepsilon'$-global for all $t_1, \ldots, t_k\in \mathcal B_{\delta}(\bar t)$.
\end{lemma}
\begin{proof}
We provide the proof for $k=2,$ but the argument holds true for other finite values of $k$. By Theorem \ref{thm:global_perturb}, $\ell(\cdot, t_1)$ can be assumed to be $\varepsilon''$-global for $\varepsilon'>\varepsilon''>\varepsilon.$ Therefore, without loss of generality, we assume that $t_1=\bar t,$ and then aim to prove that for a given $\lambda\in [0,1]$ there exists $\delta>0$ such that $\lambda \ell(x,\bar t)+(1-\lambda) \ell(x, t)$ is $\varepsilon'$-global.

We introduce $\mathfrak r(x,t) = \lambda \ell(x,\bar t)+(1-\lambda) \ell(x, t)$ and note that $\mathfrak r(x,\bar t) = \ell(x,\bar t),$ which means that $\mathfrak r(\cdot,\bar t)$ is $\varepsilon$-global and satisfies Assumption \ref{asm:contin}. Theorem \ref{thm:global_perturb} applied to $\mathfrak r$ yields that there exists $\delta>0$ such that $\lambda \ell(x,\bar t)+(1-\lambda) \ell(x, t)$ is $\varepsilon'$-global.
\end{proof}

\begin{proposition}\label{thm:uniform}
Given $\lambda_i>0$ and $\sum_{i=1}^k\lambda_i = 1,$
if for all the values of $t$ in a compact set $\mathcal C\subset \mathbb R^m$ and some $\varepsilon>0$ the function $\ell(\cdot, t)$ satisfying Assumption \ref{asm:contin} (restricted from $\mathbb R^m$ to $\mathcal{C}$) is $\varepsilon$-global, then for any $\varepsilon'>0$ such that $\varepsilon'>\varepsilon$ there exists $\delta>0$ for which any convex combination $\lambda_1\ell(\cdot, t_1)+\cdots+\lambda_m\ell(\cdot, t_m)$ is $\varepsilon'$-global for all $t_1, \ldots, t_m$ such that $\|t_i-t_j\|<\delta.$
\end{proposition}
\begin{proof}
To prove by contradiction, suppose that there exists $\varepsilon'>0$ such that $\varepsilon'>\varepsilon$ and for all $\delta>0$ there are $t_1(\delta), \ldots, t_k(\delta)$ such that $\|t_i-t_j\|<\delta$ and $\lambda_1\ell(\cdot, t_1)+\ldots+\lambda_m\ell(\cdot, t_m)$ is not $\varepsilon'$-global. From the sequence $t_1(\frac{1}{l}),$ one can extract a converging sub-sequence $t_1(\delta_l)$ since $\mathcal C$ is compact. By Lemma \ref{thm:global_convcomb}, for the point $\bar t=\lim_{l\to \infty}t_1(\delta_l)$ there exists $\delta'>0$ such that for all $t_1, \ldots, t_k\in \mathcal B_{\delta'}(\bar t)$ the convex combination $\lambda_1\ell(\cdot, t_1)+\ldots+\lambda_k\ell(\cdot, t_k)$ is $\varepsilon'$-global. Therefore, for the $t_1(\delta_l), \ldots, t_k(\delta_l)$ such that $\delta_l<\frac{\delta'}{2k}$ and $\|t_1(\delta_l)-\bar t\|<\frac{\delta'}{2k},$ this convex combination is $\varepsilon'$-global, which is a contradiction.
\end{proof}

\begin{proposition}[Formal statement of Proposition \ref{thm:close_LQR}]\label{thm:close_LQR_re}
Consider the instances of LQR such that their parameters $A, B, Q$ and $R$ belong to a compact set. Assume that none of them produces a cost function $C(W)$ with a singular Hessian at a stationary point.
For any $\varepsilon>0$ and $k\in \mathbb N,$ there exists $\delta>0$ such that for any set of $k$ considered instances of LQR defined through the parameters $\{(A_i, B_i, Q_i, R_i)\}_{i=1}^k$ with
\[\|A_i-A_j\|+\|B_i-B_j\|+\|Q_i-Q_j\|+\|R_i-R_j\|\le \delta,\]
for all $i,j\in \{1\ldots m\},$ the MAML objective \eqref{eqn:maml_obj} is $\varepsilon$-global.
\end{proposition}
\begin{proof}
Take $t=(A,B,Q,R).$ By Theorem \ref{thm:global_maml}, every function $C(W, t)$ satisfies Assumption \ref{asm:contin}. Therefore, $C(W, t)$ satisfies all the assumptions of Proposition \ref{thm:uniform} and the proof follows immediately.
\end{proof}

Given a matrix $\bar W$, the system of equations 
\begin{equation}\label{eqn:singhess}
\begin{cases}
\nabla C(W)=0 \\
\text{det}\left(\nabla^2 C(W)\right) = 0
\end{cases}    
\end{equation}
is satisfied by the parameters $A, B, Q$ and $R$ that correspond to LQRs with $\bar W$ as a stationary point with singular Hessian. All the LQR tasks that satisfy this system of equations are denoted by $HS(W).$ All the systems that have at least one stationary point with a singular Hessian are contained in $\cup_W HS(W).$ Now, notice that the union occurs over an $n\times n$ dimensional space, while $HS(W)$ is defined with a system of $(n\times n)+1$ equations. This implies that $\cup_W HS(W)$ is still a low-dimensional manifold and thus almost no LQR has a stationary point with a singular Hessian. Thus, the compact domain that is mentioned in the assumption of Proposition \ref{thm:close_LQR_re} can be a large closed set with $\cup_W HS(W)$ excluded together with its small neighborhood. Thus, Proposition \ref{thm:close_LQR} is an informal restatement of Proposition~\ref{thm:close_LQR_re}.

\subsection* {Details on the experiments}

For the numerical experiments, we computed the MAML objective explicitly using the formulas provided in the Section on LQR.

To simplify the visualization, all of the examples and counterexamples in the paper were given for one-dimensional LQR systems. As a result, the parameters $A, B,Q$ and $R$ were scalar values, and so were the state and the action. The initial state $s_0$ was chosen to be deterministic. For reproducibility, the table below collects the parameters used to construct each of the examples.

\begin{center}\label{tab:experiment}
\begin{tabular}{ | c | c | c | } 
\hline
Figure & LQR ($A, B, Q, R, s_0 s_0^\top$) & $\eta$ \\ 
\hline
\hline
Figure \ref{img:maml_single_task} & $(1, 1, 2, 2, 1)$ & $0.01$ \\ 
\hline
Figure \ref{img:maml_two_tasks} & $(1, 1, 2, 2, 1), (1, 1, 0.1, 0.1, 1)$ & $0.01$ \\ 
\hline
Figure \ref{img:maml_norm} &
\shortstack{$(1.01, 1, 1, 1, 1),$ $(1, 1.01, 1, 1, 1),$\\ $(1, 1, 1.01, 1, 1),$ $(0.99, 1, 1, 1, 1)$} & $0.01$ \\ 
\hline
Figure \ref{img:distant_afew} & $(1, 1, 1, 2, 1),$ $(1, 1, 2, 1, 1)$ & $0.1$ \\ 
\hline
Figure \ref{img:distant_several} & \shortstack{$(1, 1, 1, 1, 1),$ $(1, 1, 1, 2, 1),$\\ $(1, 1, 2, 1, 1),$ $(1, 1, 2, 3, 1),$\\ $(1, 1, 3, 2, 1)$} 
& $0.1$ \\
\hline
Figure \ref{img:distant_many} & \shortstack{$(1, 1, 1, 1, 1),$ $(1, 1, 1, 2, 1),$\\ $(1, 1, 2, 1, 1),$ $(1, 1, 2, 3, 1),$\\ $(1, 1, 3, 2, 1),$ $(1, 1, 3, 1, 1),$\\ $(1, 1, 1, 3, 1),$ $(1, 1, 4, 1, 1),$\\ $(1, 1, 1, 4, 1),$ $(1, 1, 5, 3, 1),$\\ $(1, 1, 3, 5, 1)$} & $0.1$ \\ 
\hline
\end{tabular}
\end{center}

\end{document}